\documentclass[runningheads]{llncs}
\usepackage[T1]{fontenc}
\usepackage{graphicx}
\usepackage{breakcites}
\usepackage[colorlinks=true,breaklinks=true]{hyperref}
\usepackage{url}
\usepackage{color}

\usepackage{booktabs}   %
\usepackage{subcaption} %
\usepackage{listings}
\usepackage{relsize}
\usepackage{enumerate}
\usepackage{float}
\usepackage{nicefrac}
\usepackage{multirow}
\usepackage{microtype}
\usepackage{wrapfig}
\usepackage{amsmath}
\usepackage{amssymb}
\usepackage{mathtools}
\usepackage{xcolor}

\usepackage{wrapfig}
\usepackage{algorithmic}
\usepackage[framemethod=tikz]{mdframed}
\usepackage{gensymb}
\usepackage[tableposition=below]{caption}
\usepackage[algosection,ruled,vlined,linesnumbered]{algorithm2e}
  \SetKwInput{KwIn}{Inputs}
  \SetKwFunction{Frefine}{Refine}
  \SetKwFunction{Ffilter}{Filter}
  \SetKwFunction{Fviz}{Visualize}
  \SetKwFunction{Fverify}{Verify}
  \SetKwFunction{Fpre}{GetPrecondition}
  \SetKwFunction{Fexplain}{Explain}
  \SetKwFunction{Foracle}{Oracle}
  \SetKwFunction{Fwp}{WeakestPre}
  \SetKwFunction{Fcomp}{Compare}
  \SetKwProg{Fn}{}{:}{}
  \DontPrintSemicolon
\usepackage{stmaryrd}
\usepackage[subtle,paragraphs=normal]{savetrees}

\DeclareMathSymbol{\mlq}{\mathord}{operators}{``}
\DeclareMathSymbol{\mrq}{\mathord}{operators}{`'}
\DeclareMathOperator*{\argmin}{arg\,min}
\renewcommand{\qed}{\hfill\blacksquare}

\newcommand{\conspec}{\texttt{Con}_{\texttt{spec}}}
\newcommand{\strength}{\texttt{>}}

\newcommand{\encode}[1]{#1_{enc}}
\newcommand{\head}[1]{#1_{head}}

\newcommand{\truel}{\texttt{True}}
\newcommand{\falsel}{\texttt{False}}
\newcommand{\concl}{hasCon}
\newcommand{\outl}{predict}

\newcommand{\labell}{c}
\newcommand{\Labels}{C}
\newcommand{\Concepts}{Concepts}
\newcommand{\Vars}{Vars}
\newcommand{\repmap}{r_{map}}

\newcommand{\sem}[1]{{\llbracket #1 \rrbracket}}
\newcommand{\cosine}{cos}

\newcommand{\crep}{rep}
\newcommand{\CRep}{Rep}
\newcommand{\fun}{\rightarrow}
\newcommand{\argmax}{argmax}

\newcommand{\condir}[1]{\overline{#1}}
\newcommand{\Caption}{cap}

\newcommand{\Dtrain}{D_{\text{train}}}
\newcommand{\Dtest}{D_{\text{test}}}
\begin{document}
\title{Concept-based Analysis of Neural Networks via Vision-Language Models}

\newcommand{\susmit}[1]{\textcolor{blue}{#1}}
\newcommand{\nina}[1]{\textcolor{orange}{nina: #1}}
\newcommand{\roy}[1]{\textcolor{purple}{Anirban: #1}}

\author{
Ravi~Mangal\inst{1}
\and
Nina~Narodytska\inst{2}
\and
Divya~Gopinath\inst{3}
\and
Boyue~Caroline~Hu\inst{4}
\and
Anirban~Roy\inst{5}
\and
Susmit~Jha\inst{5}
\and
Corina~P\u{a}s\u{a}reanu\inst{1,3}
}
\institute{Carnegie Mellon University\\
\and VMware Research\\
\and NASA Ames\\
\and University of Toronto \\
\and SRI International
}

\maketitle              %
\begin{abstract}
The analysis of vision-based deep neural networks (DNNs) is highly desirable but it is very challenging due to the difficulty of expressing formal specifications for vision tasks and the lack of efficient verification procedures. In this paper, we propose to leverage emerging multimodal, vision-language, foundation models (VLMs) as a lens through which we can reason about vision models. VLMs have been trained on a large body of images accompanied by their textual description, and are thus implicitly aware of high-level, human-understandable concepts describing the images. We describe a logical specification language $\conspec$ designed to facilitate writing specifications in terms of these concepts. To define and formally check $\conspec$ specifications, we build a map between the internal representations of a given vision model and a VLM,  leading to an efficient verification procedure of natural-language properties for vision models.
We demonstrate our techniques on a ResNet-based  classifier trained on the RIVAL-10 dataset using CLIP as the multimodal model.
\end{abstract}

\setcounter{footnote}{0} \section{Introduction}
\label{sec:intro}

Deep neural networks (DNNs) are increasingly used in safety-critical systems as perception components processing high-dimensional image data~\cite{esteva2021deep,janai2020computer,kaufmann2023champion,beland2020towards}. The analysis of these networks is highly desirable but it is very challenging due to the difficulty of expressing formal specifications about vision-based DNNs. This in turn is due to the low-level nature of pixel-based input representations that these models operate on and to the fact that DNNs are also notoriously opaque---their internal computational structures remain largely uninterpreted. There are also serious scalability issues that impede formal, exhaustive verification: these networks are very large, with thousands or million of parameters, making the verification problem very complex.

To address these serious challenges, our main idea is to leverage emerging multimodal, vision-language, foundation models (VLMs) such as CLIP~\cite{pmlr-v139-radford21a} as a {\em lens} through which we can reason about vision models. VLMs can process and generate both textual and visual information as they are trained for telling how well a given image and a given text caption fit together.  We believe that VLMs offer an exciting opportunity for the formal analysis of vision models, as they enable the use of natural language for probing and reasoning about visual data.

\begin{figure*}
\centering
\includegraphics[scale=0.25]{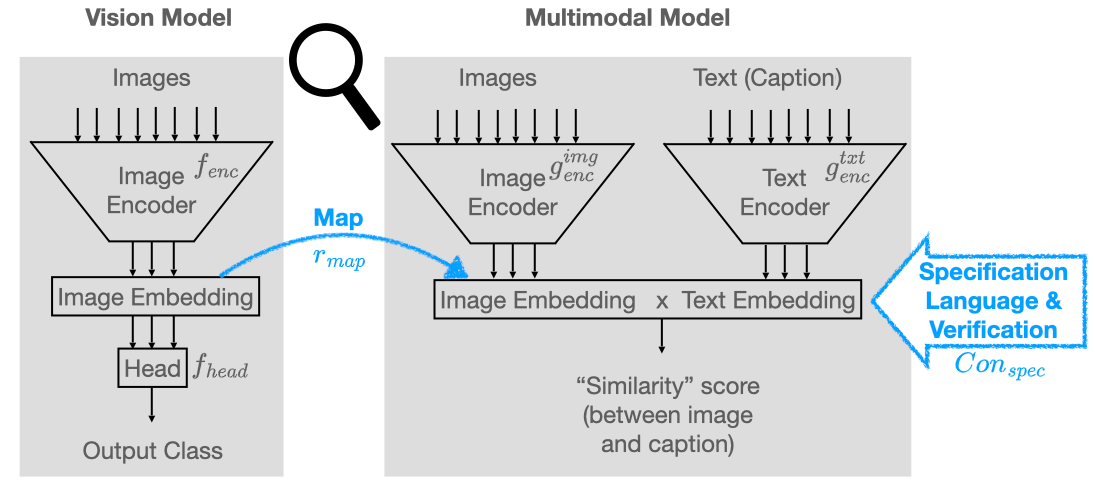}
\caption{Overview of Approach 
}
\label{fig:framework}
\end{figure*}

We illustrate our approach through Fig.~\ref{fig:framework}. 
The vision model under analysis is depicted on the left-hand side. It is an image classifier that takes as input images and it produces classes. The DNN can be seen as the composition of an {\em encoder} responsible for translating low-level inputs (pixels composing images) into high-level representations and a {\em head} that makes predictions based on these representations.  The VLM is depicted on the right-hand side. It consists of two encoders, one for each modality (text vs. image), that map inputs to the same representation space. The output of the VLM is a {\em similarity score}\footnote{Measured using cosine similarity for CLIP.} between the image and the text. To reason about the vision model, we describe a {\em logical specification language} $\conspec$ in terms of natural-language, human-understandable {\em concepts}.  For instance, for a vision model tasked with distinguishing between given classes like \textit{cat, dog, bird, car,} and \textit{truck}, relevant concepts could be \textit{metallic, ears,} and  \textit{wheels}. Domain experts and model developers can then express DNN specifications using such concepts. 
A key ingredient of our language is a strength predicate (denoted by $\strength$) defined between two concepts, which is meant to capture order relationships between concepts for a given class. For instance, if the output of the model is class \textit{truck}, we expect that the model to be more aware of concepts \textit{metallic} and \textit{wheels} rather than \textit{ears}, written as {\em metallic $\strength$ ears} and {\em wheels $\strength$ ears}, as {\em ears} are typically associated with an animal.

To define a semantics for this specification language, we could leverage previous work on {\em concept representation analysis}~\cite{Zhou_2018_ECCV,pmlr-v80-kim18d,DBLP:series/faia/YehKR21} which aims to extract meaningful concept representations, often in the form of {\em directions} (of vectors), from the latent spaces of neural networks, that are easy to understand and process. In practice,
inferring the concept representations learned by a DNN can be challenging, as finding the direction corresponding to a concept requires data that is manually annotated with truth
values of concept predicates. It can also be the case that the concepts are entangled and that no such directions exist.  Instead we propose the use of foundation vision-language models (VLMs),  as tools for concept-based analysis of vision-based DNNs. As VLMs have been trained on a large body of images accompanied by their textual description, they are implicitly aware of high-level concepts describing the images. 

We leverage recent work~\cite{moayeri2023text} to build an {\em affine map} between the image embedding space of the vision-based DNN and the corresponding image embedding space of the VLM. As a consequence, the image representation space of a VLM can serve as a proxy for the representation space of a separate vision-based DNN and can answer $\conspec$ queries about vision-based DNNs, which we encode via the textual embeddings.  %
Checking that an image satisfies certain properties reduces to checking similarity between the image representations and (logical combinations) of $\conspec$ predicates encoded in the textual space. Importantly, the verification is performed in the common text/image representation space
and it is thus {\em scalable}. The same techniques can be applied to the VLM itself (sans the map), to check properties of image representations with respect to natural-language properties, encoded in textual representations.

To demonstrate our techniques for analyzing vision-based DNNs, we consider a performant and complex ResNet-based classifier trained on the RIVAL10 dataset~\cite{moayeri2022comprehensive}. RIVAL10 is a subset of ImageNet~\cite{imagenet}, restricted to 10 classes. 
Importantly, RIVAL10 images come with annotated attributes 
which we leverage to define candidate properties to verify. In practice, we expect domain experts to provide the properties, similar to all work on verification. We use the popular vision-language CLIP model as the VLM in our analysis.

\section{Preliminaries}
\label{sec:prelim}

\paragraph{\textbf{Neural Network Classifiers.}}
A neural network classifier $f:X\rightarrow Y$ is a function where $X$ is typically a high-dimensional space $\mathbb{R}^d$ over real vectors and $Y$ is $\mathbb{R}^{|\Labels|}$ where $\Labels$ is a discrete set of labels or classes. For classification, the output
defines a score (or probability) across $|\Labels|$ classes, and the class with the highest score is output as the prediction.
We follow previous work~\cite{pmlr-v80-kim18d,Zhou_2018_ECCV,crabbe2022concept,moayeri2023text} and assume that neural classifiers can be decomposed into an encoder $\encode{f}:X\rightarrow Z$ and a head $\head{f}:Z\rightarrow Y$ where $Z$, typically $\mathbb{R}^{d'}$, is the \emph{representation (or embedding) space} (Fig.~\ref{fig:framework}). The encoder translates low-level inputs (for instance, pixels in an image) into high-level features or representations, and the head then chooses the appropriate label based on these representations. For instance, a convolutional neural network typically comprises of a sequence of convolutional layers followed by fully-connected layers---while convolutional layers (that act as the encoder) are responsible for extracting features from the inputs, the fully-connected layers (that act as the head) are responsible for classification based on the extracted features. Note that the encoder and the head are each further decomposed into a number of linear and non-linear operations (referred to as \emph{layers}). However, the internal structure of encoders and heads are not relevant for the purpose of this paper, so we treat them as black-boxes unless noted otherwise. The  \emph{embedding of an input} $x$ (also referred to as \emph{representation} or \emph{encoding} of an input) is $\encode{f}(x)$.

\paragraph{\textbf{Cosine Similarity.}}
Cosine similarity is a measure of similarity between two non-zero vectors. 
Given two n-dimensional vectors, $a$ and $b$, their cosine similarity is defined as: 
$\cosine(a,b)=\frac{a \cdot b}{||a|| \, ||b||}= \frac{\sum_i a_i b_i}{\sqrt{\sum_i a_i^2} \cdot \sqrt{\sum_i b_i^2}}$
Here $a_i$ and $b_i$ denote the $i^{th}$ components of vectors $a$ and $b$, respectively. The resulting similarity ranges from -1 meaning exactly opposite, to 1 meaning exactly the same, with 0 indicating orthogonality.

\paragraph{\textbf{Vision-Language Models.}} A vision-language model (VLM) $g$ consists of two encoders---an image encoder $\encode{g}^{img}:X\rightarrow Z$ from image inputs to some representation space $Z$ and a text encoder $\encode{g}^{txt}:T \rightarrow Z$ from textual inputs to the same representation space $Z$ (Fig.~\ref{fig:framework})%
VLMs are trained on data consisting of image-caption pairs such that, for each pair, the representation of the image and corresponding caption are as similar as possible (measured via e.g., $\cosine$ similarity) in common %
space $Z$.

VLMs such as CLIP~\cite{pmlr-v139-radford21a} can be used for a number of downstream tasks such as image classification, visual question answering, and image captioning.
By default, CLIP can be used to select the caption (from a set of captions) which has the highest similarity with a given image. This strategy can be leveraged for {\em zero-shot classification} of images~\cite{pmlr-v139-radford21a} (which we use to design $\head{g}$). For instance, given an image from the ImageNet dataset, the label of each ImageNet class, say \textit{truck}, is turned into a caption such as \textit{``An image of a truck''}. Then one can compute the cosine similarity between the embedding of the given image with the text embeddings corresponding to the captions constructed for each class, and pick the class that fits the image the best (i.e., it has the highest similarity score).
Formally, let $x \in X$ be an input image and let $t^{\labell_1}, t^{\labell_2}, ..., t^{\labell_n} \in T$ be the captions or sentences created for the classes $\labell_1, \labell_2, .. \labell_n$ relevant to a classification task. 
The zero-shot classifier returns class $\labell_k$ if and only if $\cosine(\encode{g}^{img}(x),\encode{g}^{txt}(t^{\labell_k}))\geq\cosine(\encode{g}^{img}(x),\encode{g}^{txt}(t^{\labell_k'})), \forall \labell_k'\neq \labell_k.\footnote{Tie-breaks are randomly broken.}$

Previous work indicates that this approach has high zero-shot performance~\cite{pmlr-v139-radford21a}, even though the model might not have been trained on any examples of the classes in the relevant dataset. In our work, we extend this kind of reasoning to arbitrary concepts, i.e., we construct captions about concepts instead of output classes. The same zero-shot classification procedure as before can then be used to decide if a concept is associated with an image. Note that the output classes themselves can be considered as concepts. %

\section{$\conspec$ Specification Language}
\label{sec:spec}

We present a first-order specification language, $\conspec$, that can be used to express concept-based specifications about neural classifiers. Our language makes it possible for developers to express specifications about vision models using human-understandable predicates and have such specifications be checked in an automated fashion.

\subsection{Syntax}
\label{sec:syn}

\begin{figure}[t]
        \[
                \begin{array}{@{}r@{\;}r@{\;}c@{\,}l@{}}
                        \mbox{(variables)} & x &\in& \Vars \\
                        \mbox{(concept names)} & con_1,con_2 &\in& \Concepts \\
                        \mbox{(classes)} & \labell &\in& \Labels 
                \end{array}
        \]
        \[
                \begin{array}{rcl}
                 E  &::=&  
                 \strength(x,con_1,con_2) \mid \outl(x,\labell)
                 \mid \neg E \mid E \wedge E \mid E \vee E 
                \end{array}
        \]
\vspace{-0.5cm}\caption{$\conspec$ syntax}\vspace{-0.3cm}
\label{fig:lang_syn}
\end{figure}

Fig.~\ref{fig:lang_syn} shows the $\conspec$ syntax. 
The set $\Vars$ is the set of all possible variables. 
$\Concepts$ is the set of concept names and $\Labels$ is the set of classification labels. Both $\Concepts$ and $\Labels$ are defined in a task-specific manner. 
The language defines the predicate $\strength(x,con_1,con_2)$, called a {\em strength predicate}, to express concept-based specifications where $con_1,con_2$ are constants from the set $\Concepts$. Note that since $con_1,con_2$ refer to constants, $\strength(x,con_1,con_2)$ is actually a template with a separate predicate defined for every valuation of the constants. $\outl(x,\labell)$ constrains the output of a neural classifier.

We also find it useful to define the predicate $\concl(x,con)$ to encode that input $x$ contains the concept $con$, encoded as follows:
$$
\concl(x,con) := \bigwedge_{con_i \in \Concepts~\wedge~con_i \neq con} \strength(x,con,con_i).
$$
In practice, when defining $\concl$, we typically restrict the set $\Concepts$ to only include concepts irrelevant to $x$.

\paragraph{\textbf{Example.}} Consider the five-way classifier over classes \textit{cat, dog, bird, car,} and \textit{truck}.  We assume that a domain-expert defines the set $\Concepts$. Assume that  $\Concepts$ contains \textit{metallic, ears,} and  \textit{wheels}; this does not necessarily encode  all the possible concepts that are relevant for a task. Using $\conspec$, we can write down the specification $\concl(x,metallic)~ \wedge~ \allowbreak \concl(x,wheels) ~ \wedge ~ \allowbreak \neg \concl(x,ears)$ $ \implies (\outl(x,truck) ~\vee~ \outl(x,car))$ which states that when an input image $x$ contains concepts \textit{metallic} and \textit{ears} but does not contain concept \textit{ears}, amongst the five classes, the classifier should either output \textit{truck} or \textit{car}. We can also state another specification $\outl(x,truck) \implies \strength(x,wheels,ears)$ which ensures that when the model predicts \textit{truck} it ought to be the case that the concept \textit{wheels} is more strongly present in the image than \textit{ears}, i.e., the models predicts \textit{truck} for the ``right'' reasons. Thus, predicates of the form  $\strength(x,wheels,ears)$ can be seen as {\em formal explanations} for the model decision on class \textit{truck}.

\subsection{Semantics}
\label{sec:sem}

\begin{figure}[t]
        \[
                \begin{array}{@{}r@{\;}r@{\;}c@{\,}l@{}}
                        \mbox{($\conspec$ expressions)} & e & \in & E \\
                        \mbox{(classifiers)} & f & \in & F:=\mathbb{R}^{d} \fun \mathbb{R}^{|\Labels|} \\
                        \mbox{(inputs)}&v & \in & X := \mathbb{R}^{d} \\
                        \mbox{(concept representation maps)}&\crep & \in & \CRep := \Concepts \fun  (\mathbb{R}^{d} \fun \mathbb{R}) \\
                        \mbox{(semantics)} & \sem{e} & \in & F \times X \times \CRep \fun \{\truel,\falsel\}\\
                        \end{array}
        \]
        \[
                \begin{array}{@{}r@{\;}c@{\,}l@{}}
                        \sem{\strength(x,con_1,con_2)}(f,v,\crep) & := & \crep(con_1)(v) >  \crep(con_2)(v) \\  
                        \sem{\outl(x,\labell)}(f,v,\crep) & := & (\argmax(f(v))=\{c\})\\
                        \sem{\neg e}(f,v,\crep) & := & \neg \sem{e}(f,v,\crep)  \\
                        \sem{e_1 \wedge e_2}(f,v,\crep) & := & \sem{e_1}(f,v,\crep) \wedge \sem{e_2}(f,v,\crep)  \\
                        \sem{e_1 \vee e_2}(f,v,\crep) & := & \sem{e_1}(f,v,\crep) \vee \sem{e_2}(f,v,\crep)  \\
                \end{array}
        \]
\vspace{-0.4cm}\caption{$\conspec$ semantics}\vspace{-0.5cm}
\label{fig:lang_sem}
\end{figure}

Every specification in $\conspec$ is interpreted over a triple, namely, the classifier $f \in \mathbb{R}^d\fun\mathbb{R}^{|\Labels|}$ under consideration, an input $v \in \mathbb{R}^d$, and a concept representations map $\crep \in \Concepts \fun  (\mathbb{R}^{d} \fun \mathbb{R})$; $\crep$ maps each concept $con \in \Concepts$ to a function of type $\mathbb{R}^{d} \fun \mathbb{R}$ that takes in an input and returns a number indicating the \emph{strength} at which concept $con$ is contained in the input. Although we describe some possible implementations of $\crep$ in this paper (see Defn.~\ref{def:con_rep_vlm} and~\ref{def:con_eqv}), in general, the semantics of $\conspec$ allow any possible appropriately typed implementation of $\crep$---for instance, the function corresponding to each concept could itself be a neural network~\cite{toledo2023deeper}. 

Given a concept representations map, 
the predicate $\strength(x,con_1,con_2)$ evaluates to $\truel$ if, as per $\crep$, strength of concept $con_1$ in the input is greater than that of concept $con_2$. Predicate $\outl$ constrains the class predicted by the classifier under consideration on the given input. The semantics of the logical connectives are standard.

Given these semantics for $\conspec$, we say that a classifier $f:X\fun Y$ \emph{satisfies} a $\conspec$ specification $e$ with respect to a concept representations map $\crep$ and an input scope $B$ (where $B \subseteq X$), if $e$ evaluates to $\truel$ for all inputs in $B$. Defn.~\ref{def:sat} expresses this formally. 
Here $B$ could be the whole set $X$, or a {\em region} in $X$, or just a set of (in-distribution) test images.
To simplify the notation, we write predicate $\strength(x,con_1,con_2)$ as $con_1 \strength con_2$, where $B$ is understood from the context, and $x$ is understood to range over $B$. Similarly, we write $\outl(c)$ instead of $\outl(x,c)$.

\begin{definition}[Satisfaction of specification by model]
\label{def:sat}
Given a model $f: X \fun Y$, a concept representations map $\crep$, and an input scope $B\subseteq X$, $(f,\crep,B)$ satisfies a $\conspec$ specification $e$ (denoted as $(f,\crep,B) \models e$) if,

$$\forall x \in B.~\sem{e}(f,x,\crep) = \truel.$$
\end{definition}

\section{Vision-Language Models as Analysis Tools}
\label{sec:multimodal}

The parametric nature of the $\conspec$ semantics with respect to the concept representations map $\crep$ necessitates giving a concrete implementation of $\crep$ before we can verify if a model satisfies a $\conspec$ specification. The map $\crep$ needs to be semantically meaningful---ideally, the strength of a concept in an input as per $\crep$ should be in congruence with the human understanding. For the purpose of verification, it is also essential to ensure that $\crep$ can be encoded as efficiently checkable constraints.

We propose the use of vision-language models (VLMs) such as CLIP~\cite{pmlr-v139-radford21a} as a means for implementing $\crep$. 
VLMs are typically treated as foundation models~\cite{Bommasani2021FoundationModels} that are well-trained on vast and diverse datasets to serve as building blocks for downstream applications. As a result, these models are exposed to a variety of concepts from multiple datasets and can therefore serve as a useful repository of concept representations. The richness of the VLM embedding space also opens up the opportunity to define aggregate concepts such as \textit{``wheels and metallic''} where the logical connectives are incorporated into the concept itself instead of expressing them externally via the connectives of $\conspec$, thereby simplifying the structure of a $\conspec$ specification at the cost of more complex concepts. We plan to explore this direction more in the future.

Defn.~\ref{def:con_rep_vlm} describes the implementation that we use for the concept representation map $\crep$ that uses a VLM $g$. 
This definition assumes that concepts are represented as directions in the VLM space; however note that this notion was enforced during training. In particular, the VLM's training objective requires the pair of embeddings generated for an image-caption pair, by the image and text encoders of a VLM, to be aligned in the same direction as measured via cosine similarity. 

The direction corresponding to a concept $con$ is extracted using the text encoder $\encode{g}^{txt}$ of $g$. Following previous work~\cite{moayeri2023text}, a vector $\condir{con}$ in the direction of concept is computed as the mean of text embeddings for a set of captions (denoted as $\Caption(con)$) that all refer to $con$ in different ways. We compute the mean since there may be many different valid captions constructed for the same concept and each caption leads to a slightly different embedding. Notice that there is no need for manual concept annotations to infer the concept directions. 

\begin{definition}[Linear $\crep$ via VLM]
\label{def:con_rep_vlm}
Given a VLM $g$ with image encoder $\encode{g}^{img}:X\fun Z$ and text encoder $\encode{g}^{txt}:T\fun Z$ where $Z:=\mathbb{R}^{d'}$, the linear concept representation map, $\crep$, via model $g$ is defined as,
$$
\crep(con) := \lambda x. \cos(\encode{g}^{img}(x),\condir{con}) 
$$
where $\condir{con}:=\frac{\Sigma_{t \in \Caption(con)}\encode{g}^{txt}(t)}{|\Caption(con)|}$ and $\Caption(con)$ is a set of sentences or captions referring to concept $con$.
\end{definition}

We can use Defn.~\ref{def:con_rep_vlm} to provide an implementation for $\crep$ which we statistically validate to check that it is semantically meaningful, according to human-understandable concepts (as described in Section~\ref{sec:case_studies:stat}). We can use $\crep$ to instantiate the semantics of our language and reason about properties for the VLM (i.e., CLIP) itself. In principle, we can use the same implementation of $\crep$ to reason about the original vision model. However, this would introduce a major scalability challenge---in the course of verifying a vision model $f$, we would not only need to reason about the whole model $f$ but also about the complex image encoder $\encode{g}^{img}$ of VLM $g$ since $\encode{g}^{img}$ is used to implement $\crep$. Fortunately, 
it has been previously observed~\cite{moayeri2023text} that the representation spaces of vision-based DNNs and VLMs such as CLIP can be linked to each other via an affine map. We employ this map to define a new implementation of $\crep$ which results in scalable verification.

\paragraph{\textbf{Mapping vision model embedding to VLM embedding.}} 

Let $\Dtrain \subset X $ denote an image dataset. 
Given the encoder $\encode{f}:X\fun Z_f$ of a vision model $f$ and the image encoder $\encode{g}^{img}:X\fun Z_g$ of VLM $g$,  
we aim to learn a map $\repmap:Z_f\fun Z_g$ such that the representation spaces of the two models are {\em aligned} as much as possible, i.e., given an image $x \in X$, ideally $\repmap(\encode{f}(x)) = \encode{g}^{img}(x)$ holds.

Moayeri et al. \cite{moayeri2023text} show that $\repmap$ can be restricted to the class of affine transformations, i.e., $\repmap(z) := M z + d$, where $\repmap$ is learnt by solving the 
following optimization problem,
\begin{equation}
\label{eqn:opt}
    M, d = \argmin_{M, d} \frac{1}{|\Dtrain|} \sum_{x \in \Dtrain} \|M \encode{f}(x) + d - \encode{g}^{img}(x)\|_2^2.
\end{equation}

We use $\repmap$ to define $\crep$ as follows.

\begin{definition}[Linear $\crep$ via vision model, $\repmap$, and VLM]
\label{def:con_eqv}
Given a vision model $f:X\fun Y$ with encoder $\encode{f}:X\fun Z_f$, a VLM $g$ with encoders $\encode{g}^{img}:X\fun Z_g$ and $\encode{g}^{txt}:T\fun Z_g$,  
and a linear map $\repmap$ as described above, then the concept representation map, $\crep$, is defined as,
$$
\crep(con) := \lambda x. \cos(\repmap(\encode{f}(x)),\condir{con}) 
$$
where $\condir{con}$ is a vector in $Z_g$ whose direction corresponds to concept $con$, i.e., $\condir{con}:=\frac{\Sigma_{t \in \Caption(con)}\encode{g}^{txt}(t)}{|\Caption(con)|}$.
\end{definition}

We analyze some of the properties of this $\crep$ as follows. First, let us formalize the notion of {\em faithful representation space alignment} (Defn.~\ref{def:faithful}). Intuitively, the alignment is faithful if no information is lost when using $\repmap$ to map between the two representation spaces.

\begin{definition}[Faithful alignment of representation spaces]
\label{def:faithful}
Given an encoder $\encode{f}:X\rightarrow Z_f$ of a vision model and an image encoder $\encode{g}^{img}:X\rightarrow Z_g$ of a VLM $g$, the representation space of $\encode{f}$ is faithfully aligned with the representation space of $\encode{g}^{img}$ if there exists a map $\repmap:Z_f\fun Z_g$ such that,
$$
\forall x \in X.~\repmap(\encode{f}(x)) = \encode{g}^{img}(x)
$$
\end{definition}

It is easy to see that if $\repmap$ satisfies the condition of faithful alignment, then the implementations of $\crep$ according to Defn.~\ref{def:con_rep_vlm} and \ref{def:con_eqv} coincide. Thus we could use Defn.~\ref{def:con_eqv} instead of Defn.~\ref{def:con_rep_vlm} when implementing $\crep$. However, in practice, it is possible that the representations spaces of a vision model and a VLM may not be faithfully aligned; there may be cases such that for the $\repmap$ we learn, $\repmap(\encode{f}(x)) = \encode{g}^{img}(x)$ does not hold for some images. Thus, it is possible that the two implementations of $\crep$ as we defined (one via VLM and one via vision model and $\repmap$)
may yield different results for some images. However, we argue that since we need to statistically validate $\crep$ anyway, we can directly statistically validate $\crep$ (from  Defn.~\ref{def:con_eqv}) and use it to instantiate the semantics of our language.

Once we have an implementation of $\crep$ as per Defn.~\ref{def:con_eqv}, we can define a constraint-based verification procedure for the vision model $f$. This implementation of $\crep$  has the interesting property that it allows us to factor out the constraints corresponding to $\encode{f}$ and $\encode{g}^{img}$ resulting in efficient verification.

\paragraph{\textbf{Exploiting model decomposition for verification.}} 
We exploit the decomposition of a vision model $f$ into $\encode{f}$ and $\head{f}$ to define an efficient verification procedure for vision models with respect to a $\conspec$ specification. In particular, Thm.~\ref{thm:eqv_verif} shows that the verification of a model $f$ with respect to a $\conspec$ specification $e$ can be reduced to the verification of $\head{f}$, without involving $\encode{g}^{img}$.

\begin{theorem}
\label{thm:eqv_verif}
Given a vision model $f:X\fun Y$ that can be decomposed into $\encode{f}:X\fun Z$ and $\head{f}:Z\fun Y$ where $Z:=\mathbb{R}^{d'}$, a $\conspec$ specification $e$, a linear concept representation map $\crep$, as defined in Defn.~\ref{def:con_eqv}, and an input scope $B \subseteq X$,
$$
(f,\crep,B) \models e \Leftrightarrow (\head{f},\widehat{\crep},\encode{f}(B)) \models e
$$
where $\widehat{\crep}$ operates on embeddings instead of inputs and is obtained from $\crep$ by replacing $\encode{f}$ with the identity function and $\encode{f}(B) := \{\encode{f}(x) \mid x \in B\}$.
\end{theorem}

\begin{proof}
 See Appendix~\ref{sec:app_proof}.
\end{proof}

\paragraph{\textbf{Discussion.}}
Our verification is sound only with respect to the interpretation given to the predicates by the concept representation map ($\crep$).
In this paper, we use a VLM to implement $\crep$; VLMs are trained on massive amounts of data and, as a result, are a rich repository of concepts. With the trend towards training even bigger VLMs on larger amounts of data, the implementation of $\crep$ is expected to improve.
We use a statistical analysis to validate that $\crep$ conforms with human understanding, i.e., it is semantically meaningful (described in Section~\ref{sec:case_studies:stat}).  
Further, we restrict our verification to concepts that are well represented in the VLM; if some concept is not well represented, we would not attempt verification for it.
Note also that Thm.~\ref{thm:eqv_verif} holds for arbitrary decompositions of $f$ into $\head{f}$ and $\encode{f}$. 
Our assumption, supported by previous work~\cite{moayeri2023text}, is that the head of the network is moderately sized, making it suitable for analysis with modern decision procedures. 
Although these are strong assumptions, they enable us to make progress on the challenging problem of verification of semantic properties for vision models.

\section{Case Study}
\label{sec:case_studies}
We describe the instantiation of our approach via a case study using a ResNet18 vision model~\cite{he2016deep,moayeri2022comprehensive} and CLIP~\cite{pmlr-v139-radford21a} as the VLM on the RIVAL10 image classification dataset~\cite{moayeri2022comprehensive}.

\subsection{Dataset, Concepts, and Strength Predicates}
\label{sec:case_studies:dataset}

\paragraph{RIVAL10 (\textbf{RI}ch \textbf{V}isu\textbf{AL} Attributions)}~\cite{moayeri2022comprehensive} is a dataset consisting of 26k images drawn from ImageNet~\cite{imagenet}, organized into 10 classes matching those of CIFAR10.
The dataset contains manually annotated instance-wise labels for 18 attributes (which we use to define set $\Concepts$) as well as the respective segmentation masks for these attributes on the images. 
The dataset is partitioned into a train set with 21k images and a test set with 5k images.

\begin{wrapfigure}{l}{7cm}%
    \centering
    \includegraphics[width=.9\linewidth]{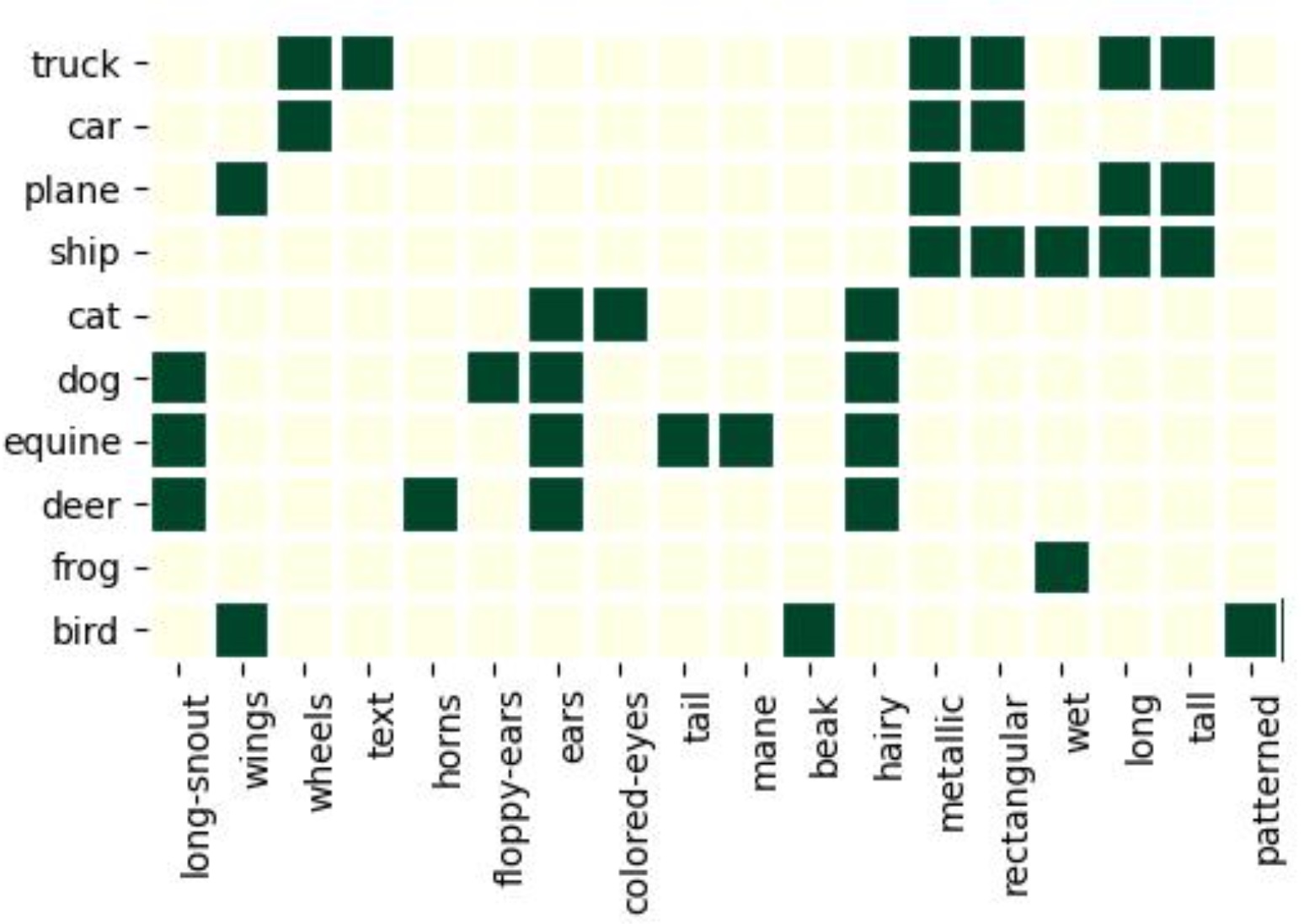}
    \caption{Relevant concepts per class}
    \label{fig:hitmap}\vspace{-0.5cm}
\end{wrapfigure}
These annotations can be used to extract the concepts that are relevant for each class and also to elicit strength predicates as follows. For each class $\labell$, we compute the percentage of inputs (from the train set) that have ground truth $\labell$ and are annotated with concept $con$. If this is greater than a threshold (70\%) we consider that concept as {\em relevant}. Figure~\ref{fig:hitmap} illustrates the relevant concepts that are computed for each class; on the y axis we list the classes, on the x axis we list the concepts; green boxes mark the relevant concepts. For instance, for class \textit{truck}, concepts  \textit{wheels, tall, long, rectangular, metallic} and \textit{text} are all considered relevant whereas the other concepts (e.g., \textit{tail} and \textit{wings}) are irrelevant. 

The information in Figure~\ref{fig:hitmap} can be further used to elicit strength predicates relevant for each class $\labell$---for each pair of relevant ($con_r$) and irrelevant ($con_{ir}$) concepts, we can define a predicate $con_r \strength con_{ir}$ that we expect to hold on inputs of class $\labell$.
For instance, for class \textit{truck}, we formulate a total of 72  predicates, corresponding to all the possible combinations of relevant and irrelevant concepts;
e.g., we elicit $wheels \strength wings$ and $metallic \strength tail$, as $wheels$ and $metallic$ are relevant concepts, whereas $wings$ and $tail$ are not.
Overall, we formulate a total of 522 such predicates for the model. In practice, we expect human domain experts to directly provide the concepts and strength predicates relevant to the task. Since we only use the concept annotations to elicit the strength predicates, there is no need to manually annotate each image if these predicates are given.

\subsection{Models}
We employed an already trained CLIP model\footnote{ \url{https://github.com/openai/CLIP}.}\cite{pmlr-v139-radford21a}, with a Vision Transformer~\cite{dosovitskiy2021an} (particularly, ViT-B/16) as the image encoder and a Transformer~\cite{vaswani2017attention} as the text encoder, as the VLM for our experiments. The representation space of the CLIP model is of type $\mathbb{R}^{512}$. The head for the VLM is a zero-shot classifier implemented as described in Section~\ref{sec:prelim}. On the RIVAL10 dataset, it has a test accuracy of $98.79\%$.  As our vision model, we employed an already trained ResNet18 model made available by the developers of the RIVAL10 dataset\footnote{\url{https://github.com/mmoayeri/RIVAL10/tree/gh-pages}} that is pretrained on the full ImageNet dataset and the final layer of the model, i.e., the head, is further fine-tuned on the RIVAL10 dataset in a supervised fashion using the class labels. The head is a single linear layer with no activation functions that accepts inputs of type $\mathbb{R}^{512}$ and produces outputs of type $\mathbb{R}^{10}$. The model has a test accuracy of $96.73\%$ on RIVAL10.

\subsection{Extraction of Concept Representations}
\label{sec:case_studies:crep}

In order to analyze the models with respect to the specifications that we formulate for RIVAL10,   we need to build the concept representation map $\crep$ using the VLM (as defined in Defn.~\ref{def:con_eqv}). In particular, we need to extract the directions corresponding to the relevant concepts in CLIP's representation space and learn the affine map from the representation space of ResNet18 to CLIP. The first step is to define the set of \emph{relevant} concepts. For RIVAL10, these are the 18 attributes defined by the developers of the datatset. In general, the relevant concepts are elicited in collaboration with domain experts.

Next, to extract the directions corresponding to these 18 concepts, we use CLIP's text encoder (referred to as $\encode{g}^{txt}$) in a manner similar to the approach described by Moayeri et al.~\cite{moayeri2023text}. In particular, as described in Defn.~\ref{def:con_rep_vlm}, for each concept $con$, we create a set of captions that refer to the concept (denoted as $\Caption(con)$). As an example, for the concept \textit{metallic}, we use captions such as \textit{``a photo containing a metallic object''}, \textit{``a photo of a metallic object''}, etc. The complete set of captions used is given in Appendix~\ref{sec:app_captions}.
We then apply $\encode{g}^{txt}$ to each caption in the set $\Caption(con)$ and compute the mean of the resulting embeddings. The direction of the resulting mean vector corresponds to $\condir{con}$, i.e., the direction representing the concept $con$ in CLIP's representation space. 

The final step is to learn the affine map $\repmap$ from the representation space of the ResNet18 vision model to the representation space of the CLIP model used in our case-study by solving the optimization problem in Equation~\ref{eqn:opt}. The problem is solved via gradient descent, following the approach of Moayeri et al.~\cite{moayeri2023text}, on RIVAL10 training data for $50$ epochs using a learning rate of $0.01$, momentum of $0.9$, and weight decay of $0.0005$.  The learnt map, of type $\mathbb{R}^{512}\fun\mathbb{R}^{512}$,
has a low Mean Squared Error (MSE) of 0.963 and a high Coefficient of Determination (R$^2$) of 0.786 on the test data, suggesting that it is able to align the two spaces.

\subsection{Statistical Validation of $\crep$ via Strength Predicates}
\label{sec:case_studies:stat}

\begin{figure}[t]
    \centering
    \begin{subfigure}{0.49\textwidth}
        \centering
        \includegraphics[width=0.99\textwidth]{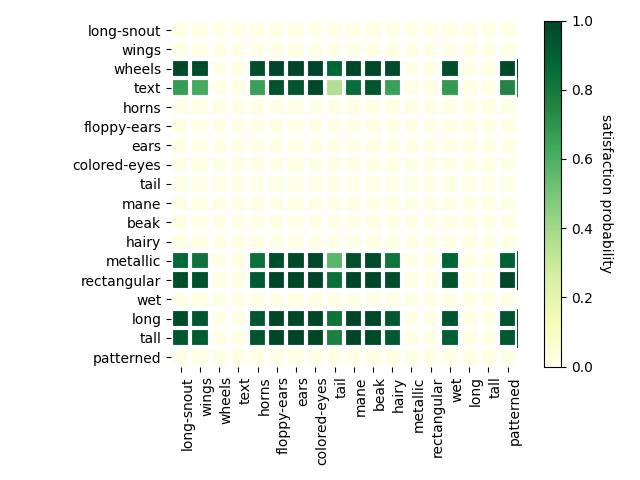} %
        \caption{Strength predicates for  \textit{truck}\vspace{-0.3cm}\label{fig:clip:stat_truck}}
    \end{subfigure}\hfill
    \begin{subfigure}{0.49\textwidth}
        \centering
        \includegraphics[width=0.99\textwidth]{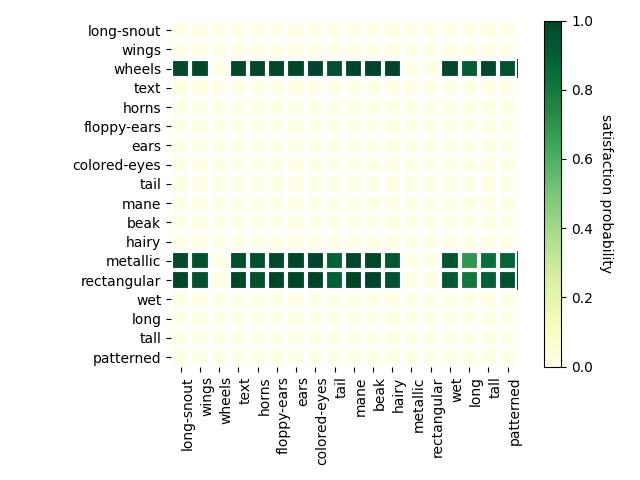} %
        \caption{Strength predicates for  \textit{car}\vspace{-0.3cm}\label{fig:clip:stat_car}}
    \end{subfigure}
    \caption{Satisfaction probabilities for $\crep$ implemented only using CLIP model $g$}
    \label{fig:clip:stat}
    \vspace{-0.5cm}
\end{figure}

\begin{figure}[t]
    \centering
    \begin{subfigure}{0.49\textwidth}
        \centering
        \includegraphics[width=0.99\textwidth]{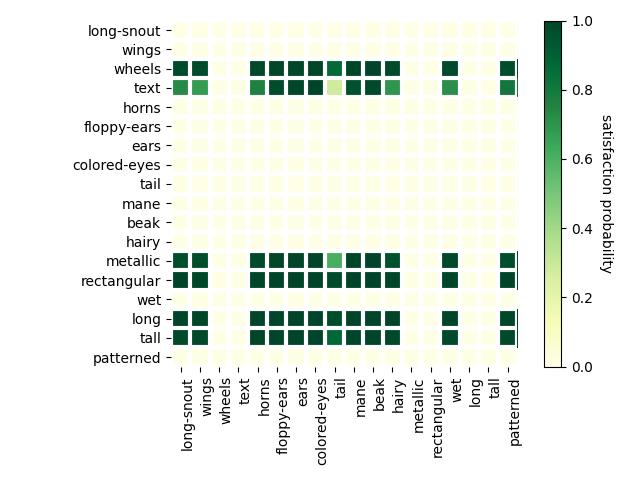} %
        \caption{Strength predicates for  \textit{truck}\vspace{-0.3cm}\label{fig:resnet:stat_truck}}
    \end{subfigure}\hfill
    \begin{subfigure}{0.49\textwidth}
        \centering
        \includegraphics[width=0.99\textwidth]{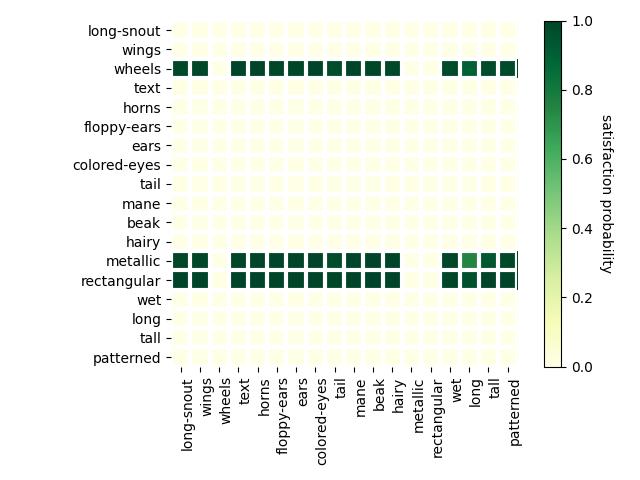} %
        \caption{Strength predicates for  \textit{car}\vspace{-0.3cm}\label{fig:resnet:stat_car}}
    \end{subfigure}
    \caption{Satisfaction probabilities for $\crep$ implemented using ResNet18 model $f$ and CLIP model $g$}
    \label{fig:resnet:stat}
    \vspace{-0.5cm}
\end{figure}

Recall that since the semantics of $\conspec$ are parameterized by the concept representation map $\crep$, the validity of the semantics is dependent on the validity of $\crep$. Our next step is to statistically validate the $\crep$   implemented using CLIP model $g$ (as per Defn.~\ref{def:con_rep_vlm}) as well the one implemented via the affine map $\repmap$ (as per Defn.~\ref{def:con_eqv}) between the ResNet18 $f$ and $g$'s representations spaces.

Towards this end, we statistically validate our formulated strength predicates using the RIVAL10 test data ($\Dtest$) for the concept representation maps under consideration. In particular, for a given $\crep$, we measure the rate (referred to as \emph{satisfaction probability}), for all classes $\labell$, at which inputs in $\Dtest$ with ground-truth label $\labell$ satisfy the strength predicates formulated for $\labell$. Our intuition is that, given an ideal concept representation map, the strength predicates for a class $\labell$ should always hold, i.e., with probability 1.0 (unless the formulated predicates are nonsensical, for instance, $hairy \strength metallic$ for class \textit{truck}, represented by yellow-colored squares in Fig.~\ref{fig:clip:stat} and~\ref{fig:resnet:stat}). Thus, consistently low satisfaction probabilities indicate a low-quality $\crep$. Note that this procedure neither requires access to an ideal concept representation map nor data with manually annotated concept labels if the strength predicates are given.

\paragraph{\textbf{Results.}}
The results for $\crep$ implemented only using CLIP model $g$ (as in Defn.~\ref{def:con_rep_vlm}) are shown in Fig.~\ref{fig:clip:stat_truck} and~\ref{fig:clip:stat_car}. %
Concepts on the Y-axis represent $cor_{r}$ and on the X-axis represent $con_{ir}$ in the evaluated strength predicates.
We see that for class \textit{truck}, except for strength predicates involving concept $\textit{text}$, the satisfaction probabilities are consistently high. Similar results are observed for class \textit{car} as well as other classes. This provides strong evidence for our hypothesis that a concept representation implemented using CLIP is of high-quality. 

The results for $\crep$ implemented using the ResNet18 model $f$ and CLIP model $g$ via the affine map (as per Defn.~\ref{def:con_eqv}) are summarized in Fig.~\ref{fig:resnet:stat_truck} and~\ref{fig:resnet:stat_car}. The trends are similar to those observed for $\crep$ implemented only using $g$. In addition to suggesting that the $\crep$ based on the affine map $\repmap$ is also of high-quality, these results are evidence in favor of the assumption that the representation space of $\encode{f}$ is well aligned with the representation space of $\encode{g}^{img}$ via our learned affine map (also suggested by the low MSE and high R$^2$ of $\repmap$).
However, note that the representation spaces are not {\em faithfully  aligned}, as indicated by the counterexamples in Fig.~\ref{fig:debug_small}. %

\begin{wrapfigure}{l}{6cm}\vspace{-0.5cm}
    \centering
    \begin{subfigure}{0.24\textwidth}
        \centering
        \includegraphics[width=0.99\textwidth]{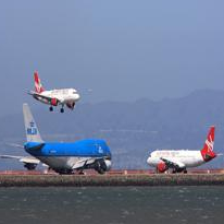} 
        \caption*{\tiny (a) Class: \textit{plane}, CLIP: \textit{plane}, ResNet: \textit{ship}, CLIP via $\repmap$: \textit{ship}\label{fig:debug_smalla}}
    \end{subfigure}\hfill
    \begin{subfigure}{0.24\textwidth}
        \centering
        \includegraphics[width=0.99\textwidth]{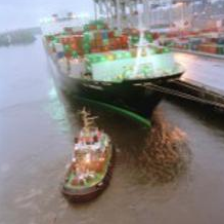} 
        \caption*{\tiny (b) Class: \textit{ship}, CLIP: \textit{ship}, ResNet: \textit{frog}, CLIP via $\repmap$: \textit{frog}\label{fig:debug_smallb}}
    \end{subfigure}
    \caption{Evidence for representation spaces of $\encode{f}$ and $\encode{g}^{img}$ not being  faithfully aligned.}
    \label{fig:debug_small}
    \vspace{-1cm}
\end{wrapfigure}

Consider Fig.~\ref{fig:debug_small}a. The ground-truth class of the input image $x$ is \textit{plane}, the ResNet18 model $f$ mis-classifies it as a \textit{ship}, and the CLIP model $g$ (using a zero-shot classification head) correctly classifies it as a \textit{plane}. However, the label assigned by the zero-shot classification head of $g$ to the embedding vector $\repmap(\encode{f}(x))$ 
is \textit{ship} and not \textit{plane}. Clearly, not only are the embeddings $\encode{g}(x)$ and $\repmap(\encode{f}(x))$ different but they are far enough from each other to cause different classification outcomes. Fig.~\ref{fig:debug_small}b presents a similar scenario. 

\subsection{Verification}
\label{sec:case_studies:verif}

Next, we use the elicited predicates to form specifications that we attempt to formally verify. %
In particular, we performed two sets of experiments---first, we conduct formal verification of the CLIP model itself, and then, we formally verify the ResNet18 model.  In both cases, we can reduce the verification to solving of linear constraints (despite the fact that $\cosine$ involves non-linear constraints). Note that the reduction to linear constraints relies on the assumption that the head of the model is a linear layer. While this assumption holds true for the models we consider, in general the head may have non-linear layers in which case we would need a more powerful solver to address the non-linear constraints.
For space reasons, we report only the ResNet18 results here, while the CLIP results are delegated to Appendix~\ref{sec:app_clip}.

We restrict our investigation only to statistically significant strength predicates, i.e., those that hold on the train set %
with over $95\%$ probability, since they are more likely to be semantically meaningful and are more likely to hold. Moreover, in order to decide if a model satisfies a specification, we need to define an input scope $B \subseteq X$ (see Defn.~\ref{def:sat}).

We could try to show that a specification holds for all inputs $v\in X$ but we find that such a requirement is too strong. Majority of the inputs in set $X$ do not correspond to any meaningful image (i.e., are out-of-distribution), and attempting to verify model behavior at such inputs typically leads to verification failure without discovering any useful counter-examples. %

Ideally, we would like to verify the models on inputs that are in-distribution, i.e., in the support set of the distribution over $X$ %
that characterizes the input data. Unfortunately, this is an unsolved challenge in the research community, as it is not known how to mathematically represent in-distribution data in a  sound and complete manner. 
We describe here approximations of the same that can be useful in practice. 
We also leverage Thm.~\ref{thm:eqv_verif} to conduct proofs in the representation space instead of the input space.

To simplify the notation, in the rest of the paper, we refer to the image of $B$ under an encoder as $\overline{B}$.
We attempt to check specifications of the form $\outl(c)\implies con_1\strength con_2$, which have the flavor of formal explanations for model behaviour.
We check the specifications via constraint integer programming, leveraging the off-the-shelf solver SCIP~\cite{bestuzheva2021scip}; i.e., we attempt to solve $\neg(\outl(c)\implies con_1\strength con_2)$ equivalent to $\outl(c) \wedge \neg(con_1\strength con_2)$. If no solution is found, it means the property holds, while a solution indicates a counterexample. As we only solve linear programs, the constraints take less than a few seconds to solve for all the specifications that we investigated.
Our experiments were performed on Intel(R) Xeon(R) Silver 4214 CPU @ 2.20GHz.
We also attempted to check other specifications where the predicate $\outl$ is on the right-hand side of implication e.g., $\concl(ears)\implies \neg \outl(truck)$. Although such properties can also be encoded as integer linear programs, for space reasons, we focus on the former type of specifications here.

\vspace{-0.5cm}
\subsubsection{Verification of ResNet18.}
For this experiment, we verify a ResNet18 vision model $f$ while using CLIP as a VLM $g$ to implement the concept representation map $\crep$ as per Defn.~\ref{def:con_eqv}. 
We first identify interesting regions in the embedding space of the network, and then check if the statistically significant concept predicates hold in these regions. 
\vspace{-0.3cm}
\paragraph{\textbf{Focus regions.}}
We start by defining an input scope (or region) $\overline{B}:=[[l_1,u_1], \ldots,\allowbreak [l_p, u_p]]$ in the embedding space of the vision model. 
We experiment with three different approaches, referred to as $A_1$,
$A_2$ and $A_3$, for defining $\overline{B}$. 
For each class $\labell$, approach \textbf{$A_1$} attempts to capture the set of in-distribution inputs which the model classifies as $\labell$. We consider inputs from the RIVAL10 data set
that the model classifies as class $\labell$ (denoted as $D^{\labell}$)
and define the focus region $\overline{B}$ for class $\labell$ as $\overline{B} := [[l_1,u_1], \ldots, [l_p, u_p]]$, where $\forall i \in [1,p] \;\; l_i = min(\{\encode{f}(x)_i\mid x \in D^{\labell}\}) \;, u_i = max(\{\encode{f}(x)_i\mid x \in D^{\labell}\})$ and $z_i$ refers to $i$th feature of embedding vector $z=\encode{f}(x)$. 

For each class $\labell$, approach \textbf{$A_2$} attempts to refine the region defined in $A_1$ by restricting the set characterizing in-distribution embeddings to a region where the model output is correct for class $\labell$. The intuition is that model is more likely to satisfy the specifications when it also predicts the correct class, while in case of mis-predictions, we expect the model to violate (at least some of) the specifications.
Given a set of correctly classified inputs for class $\labell$ 
($D'^{\labell}$) 
we define $\overline{B} := [[l_1,u_1], \ldots, [l_p, u_p]]$, where $\forall i \in [1,p] \;\; l_i = min(\{\encode{f}(x)_i\mid x \in D'^{\labell}\}) \;, u_i = max(\{\encode{f}(x)_i\mid x \in D'^{\labell}\})$. 

Given a set of input images classified to the same class, the model may internally apply different logic, possibly using different concepts, for different subsets of inputs.  
In approach \textbf{$A_3$}, we employ the method proposed in~\cite{gopinath2019property} to extract different preconditions ${r_1},\ldots,{r_n}$  (in terms of neuron-patterns) characterizing such sub-sets for each class. Let $D^{r_j}$ denote the subset of inputs that satisfy $r_j$ for class $\labell$;
we define the respective focus region, $\overline{B}^{r_j} := [[l_1,u_1], \ldots, [l_p, u_p]]$ , where $\forall i \in [1,p] \;\; l_i = min(\{\encode{f}(x)_i\mid x \in D^{r_j}\}) \;, u_i = max(\{\encode{f}(x)_i\mid x \in D^{r_j}\})$.

\paragraph{\textbf{Encoding the verification problem.}}
We want to formally check if the vision model $f$ satisfies a specification $e$ with respect to a concept representation map $\crep$ (as defined in Defn.~\ref{def:con_eqv}) and an input scope $B$ (i.e., $(f,\crep,B)\models e$) which, using Thm.~\ref{thm:eqv_verif}, can be rephrased as 
$(\head{f},\widehat{\crep},\overline{B})\models e$).

The head $\head{f}$ of the ResNet18 model comprises only of a single linear layer of type $\mathbb{R}^{512}\fun\mathbb{R}^{10}$. It is of the form $A w + b$ where $A$ and $b$ are parameters of the layer and $w$ denotes embedding vectors computed by the vision model (i.e., given an image $x$ of class $\labell$, $w=\encode{f}(x)$). 
Recall that image $x$ is classified as class $\labell$ iff 
$\head{f}(w)_{\labell} \geq  \head{f}(w)_{\labell_k}, \forall \labell_k\neq \labell.\footnote{Given vector $a$, $a_i$ denotes its $i$th element. We assume here that the set of classes $\Labels$ is $\{1,\ldots,10\}$, obtained by mapping class names to corresponding output indices.}$

These conditions for a class $\labell$ can be rewritten as 
$ 0 \leq \sum_i (A_{\labell,i} - A_{\labell_k,i}) w_i + (b_{\labell} - b_{\labell_k}), \forall \labell_k\neq \labell,$ where $A_{i,j}$ denotes $A[i][j]$th element of matrix $A$.
While such a straightforward encoding of the head is feasible for the ResNet18 model since it only has one linear layer in the head, in general, if the head has multiple layers such conditions can be obtained either via symbolic execution or by encoding the behaviour of all the head layers as constraints as in existing complete DNN verifiers~\cite{bastani2016measuring,tjeng2018evaluating}.

To define the concept representation map $\widehat{\crep}$, we first extract embeddings corresponding to concepts in the embedding space of CLIP (denoted as $q^{con_i}$ for concept $con_i$) using the CLIP text encoder.
As explained in Section~\ref{sec:multimodal}, we also learn an affine map $\repmap$ 
such that $z= \repmap(w)$ where $w$ is a vector in the vision model $f$'s embedding space and $z$ denotes a vector in the VLM model $g$'s embedding space. For simplicity, we assume that the length of both vectors is $p$ which holds in this case. 

 Equations~\eqref{eq:model:one:box}--\eqref{eq:model:one:rel} show our encoding to check if a specification $\outl(\labell) \implies con_1 \strength con_2$ holds for the region $\overline{B}$ (via the negated specification as earlier).

\vspace{-0.5cm}
\begin{align}
l_i \leq w_i\leq u_i,  [l_i,u_i] \in \overline{B}\label{eq:model:one:box}, \forall i=\{1,\ldots, p\}
\end{align}
\vspace{-0.3cm}
\begin{align}
0 \leq \sum_i (A_{\labell,i} - A_{\labell_k,i}) w_i + (b_{\labell} - b_{\labell_k}), \forall \labell_k\neq \labell  \label{eq:model:one:class}
\end{align}
\vspace{-0.3cm}
\begin{align}
z_j = \sum_i M_{j,i} w_i +d_j , i,j\in\{1,\ldots,p\} \label{eq:model:one:map}
\end{align}
\vspace{-0.3cm}
\begin{align}
\sum_i \frac{z_i}{\lVert z \rVert} \frac{q_i^{con_2}}{\lVert q^{con_2} 	\rVert} > \sum_i \frac{z_i}{\lVert z \rVert}
\frac{q^{con_1}_i}{\lVert q^{con_1} 	\rVert}\label{eq:model:one:rel}
\end{align}
\vspace{-0.3cm}

Equation~\eqref{eq:model:one:box} encodes region $\overline{B}$ constraints, Equation~\eqref{eq:model:one:class} ensures that $w$ is classified as $C$ by the vision model,
Equation~\eqref{eq:model:one:map} encodes linear mapping  $\repmap$ from $w$ to $z$, where $M$ and $d$ are parameters of the map.  Finally, Equation~\eqref{eq:model:one:rel}
encodes the negation of $con_1 \strength con_2$. 
Again, we can cancel out the norm of $z$ from both sides of the inequality simplifying it to  a linear inequality and introduce a slack variable $\varepsilon$ that we maximize to find the maximum violation:
\begin{align}\vspace{-0.5cm}
\sum_i {z_i} \frac{q_i^{con_2}}{\lVert q^{con_2} 	\rVert} > \varepsilon + \sum_i {z_i}
\frac{q^{con_1}_i}{\lVert q^{con_1} 	\rVert}\label{eq:model:one:simpl:opt:rel}
\end{align}

\begin{figure}[t]
    \centering
    \begin{minipage}{0.49\textwidth}
        \centering
        \includegraphics[width=0.99\textwidth]{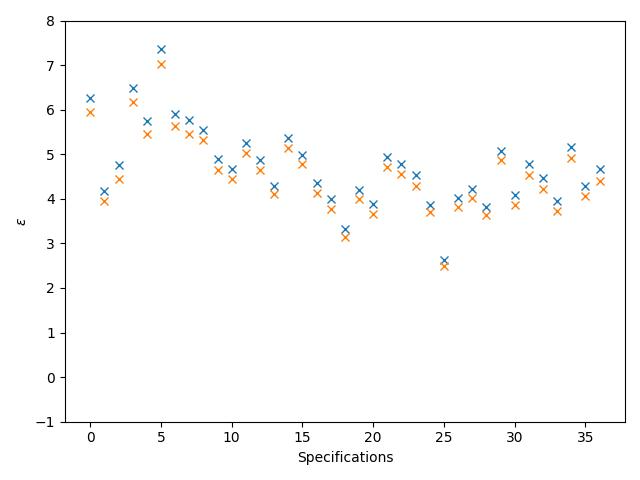} %
        \caption{Specification's violation measure $\varepsilon$ for $A_1$ method in blue and $A_2$ in orange. Class label is \textit{truck}. \label{fig:model:truck:min_max}}
    \end{minipage}\hfill
    \begin{minipage}{0.49\textwidth}
        \centering
        \includegraphics[width=0.99\textwidth]{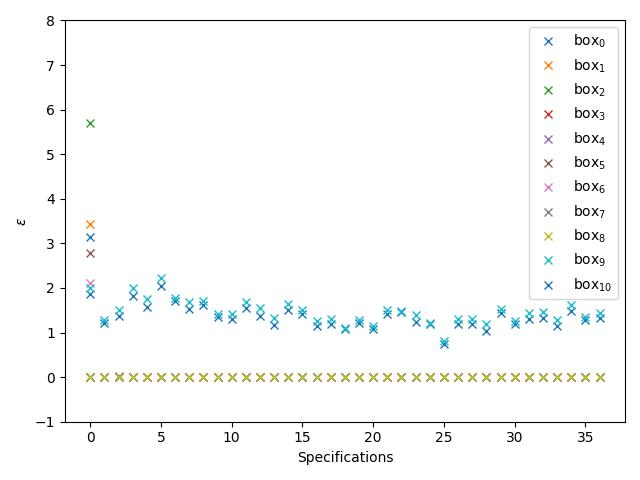} %
        \caption{Specification's violation measure $\varepsilon$ for $A_3$ method. 
        Each color corresponds to a box. Class label is \textit{truck}.   
        \label{fig:model:truck:min_max_rules}}
    \end{minipage}\vspace{-0.4cm}
\end{figure}

\begin{figure}[t]
    \centering
    \begin{minipage}{0.49\textwidth}
        \centering
        \includegraphics[width=0.99\textwidth]{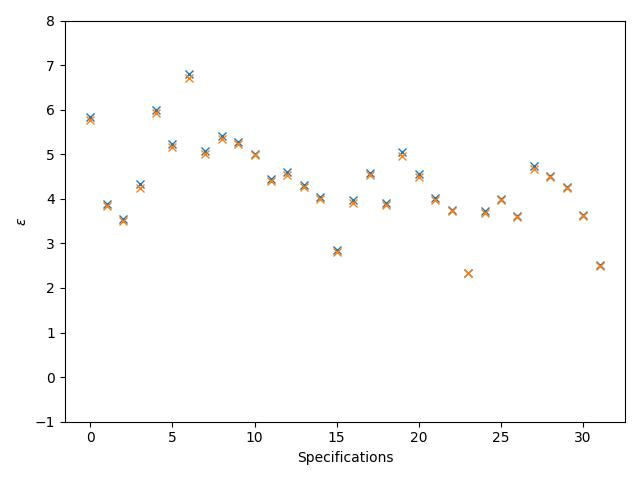} %
        \caption{Specification's violation measure $\varepsilon$ for $A_1$ method in blue and $A_2$ in orange. Class label is \textit{car}. \label{fig:model:car:min_max}}
    \end{minipage}\hfill
    \begin{minipage}{0.49\textwidth}
        \centering
        \includegraphics[width=0.99\textwidth]{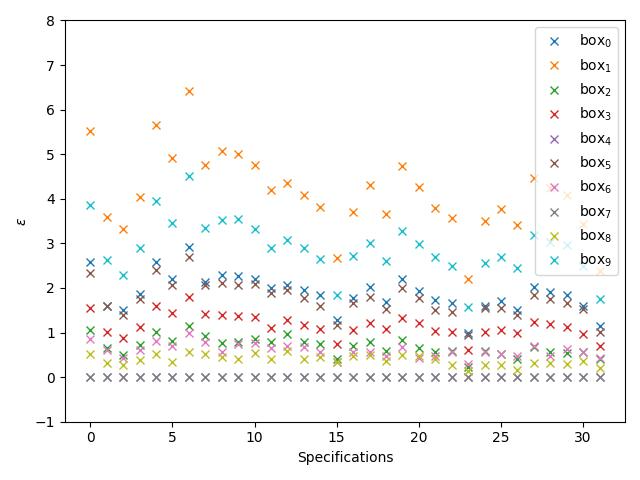} %
        \caption{Specification's violation measure $\varepsilon$ for $A_3$ method. 
        Each color corresponds to a box. Class label is \textit{car}.   
        \label{fig:model:car:min_max_rules}}
    \end{minipage}\vspace{-0.5cm}
\end{figure}

\paragraph{\textbf{Results.}}
Figures~\ref{fig:model:truck:min_max}--\ref{fig:model:car:min_max_rules} show results for our experiments for \textit{truck} and \textit{car} labels.
For \textit{truck}, Figure~\ref{fig:model:truck:min_max} shows results for options  $A_1$ and $A_2$  and Figure~\ref{fig:model:truck:min_max_rules} shows results for option $A_3$  which produces 11 boxes. 
For \textit{car}, Figure~\ref{fig:model:car:min_max} shows results for options  $A_1$ and $A_2$  and Figure~\ref{fig:model:car:min_max_rules} shows results for option $A_3$  which produces 10 boxes. 
For each class, a lower violation
of relevant specifications indicates that the model is making the classification decision for the ``right'' reasons.
For instance, consider specification 25 in Fig.~\ref{fig:model:truck:min_max} that has the lowest value for the violation measure $\varepsilon$ suggesting that if the ResNet18 model predicts \textit{truck}, it likely that %
$rectangular \strength patterned$ holds.
On the other, the high $\varepsilon$ value for specification 5 suggests that %
$wheels \strength colored-eyes$
is less likely to hold. 
The results also show that the region defined using $A_2$ has lower violations than $A_1$ for all specifications. This is expected since this region captures portions of the embedding space for which the model output is correct, and the model can be expected to satisfy more relevant specifications on correctly classified inputs vs. others. For $A_3$, overall all the boxes show lower violations than $A_1$ or $A_2$, indicating that they correspond to tighter regions capturing more precisely the inputs on which the model behaves correctly.  More interestingly, each box seems to satisfy (violate) different sets of specifications, indicating that each region corresponds to a different scenario or a different profile in terms of concepts.%
We note that for some boxes (box 7 and 8 for truck and box 7 for car),
the violation measure is 0 for all specifications. However this is because the corresponding boxes correspond to only one valid input (lower bound equals upper bound for each dimension).

\section{Related Work}
\label{sec:related}

The use of deep learning models, particularly computer vision, in safety-critical high-assurance applications such as autonomy and surveillance raises the need for formal analysis due to their complexity, and opacity. This need is compounded by the low-level, pixel-based nature of their inputs, making formal specifications challenging. Our proposed approach to leverage Vision-Language Models (VLMs) as a means to formalize and reason about DNNs in terms of natural language builds upon and intersects several research domains.

{\em \textbf{Formal Analysis of DNNs.}}
The formal analysis of DNNs, especially in safety-critical applications, has been a subject of growing interest. Works like \cite{huang2017safety, katz2017reluplex, dutta2018output} have explored methods to verify the safety and correctness of DNNs, but they often grapple with the complexity and scale of these networks. Our approach differs by translating the problem into the realm of natural language, thus potentially bypassing some of the direct complexities involved in analyzing the networks themselves. Further, existing DNN verification methods are restricted to simple robustness properties expressed in terms of input pixels while our approach enables semantic specification and verification of DNN models.

{\em \textbf{Multimodal Vision-Language Models (VLMs).}}
The development and application of multimodal VLMs have seen significant advancement, particularly models like CLIP \cite{radford2021learning} that process both text and images. These models offer a novel perspective for analyzing visual data, as they can interpret the relationships between images and text. 
This cross-modality relationship has been used for informal testing and diagnosis of vision models \cite{gao2023adaptive,eyuboglu2022domino,zhang2022diagnosing}.
Our approach is innovative in that it uses these capabilities to reason about vision models by interpreting concept representations, thereby enabling semantic formal analysis of models. Our approach is agnostic to the VLM used, and can be easily adapted to any VLM.

{\em \textbf{Concept Representation in Neural Networks.}}
The field of concept representation analysis in neural networks, as explored in \cite{Zhou_2018_ECCV, pmlr-v80-kim18d, DBLP:series/faia/YehKR21}, seeks to understand and extract interpretable features from the latent spaces of neural networks. While these approaches provide valuable insights, they often require extensive manual annotation and may struggle with entangled concepts. The disentangled concept learning~\cite{burgess2018understanding,cunningham2022principal} has also received attention recently but is limited to relatively low-dimensional data.  
By focusing on leveraging relative concept similarity, our approach side steps these challenges of decomposed concept learning and instead proposes a more scalable and automated approach to understanding concept representation in vision models. 
Our work builds on recent results on relationships between concept representation in vision and language models~\cite{pmlr-v80-kim18d,Zhou_2018_ECCV, crabbe2022concept,bai2023concept, nanda2023emergent,wang2023concept,park2023linear}.

{\em \textbf{Bridging DNN Embeddings and Natural Language.}}
Recent works like~\cite{moayeri2023text} have explored methods to correlate the embeddings of neural networks with natural language representations. This research is pivotal to our methodology, as it provides a basis for using VLMs to interpret and verify the properties of DNNs in natural language terms.

{\em \textbf{Scalability in Neural Network Verification.}}
The challenge of scalability in neural network verification is well-documented, and a number of techniques \cite{henriksen2019efficient, singh2019abstract} have been proposed to use abstraction-refinement to scale verification. Many existing methods suffer when scaling to networks with millions of parameters. Our approach aims to simplify this process by translating the verification into solving constraints within the common text/image representation space of VLMs, thus potentially offering a more scalable neuro-symbolic solution distinct from fuzzy-logic based neuro-symbolic methods \cite{donadello2017logic}.

In summary, our approach to leveraging VLMs for the formal analysis of vision-based DNNs synthesizes elements from the formal analysis of neural networks, concept representation analysis, multimodal language models, and scalable verification techniques. This integration of ideas offers a novel perspective that could address the longstanding challenges of formal analysis of DNNs.

\section{Conclusion}
\label{sec:conclusion}
We proposed the use of foundation vision-language models (VLMs) as
tools for concept-based analysis of vision-based DNNs.
We described a specification language, 
$\conspec$, to facilitate writing 
specifications in terms of human-understandable, natural-language descriptions of concepts, which are machine-checkable.
We illustrated our techniques on a ResNet classifier %
leveraging CLIP.

Our verification results serve as a demonstration that proofs are possible even for very large models such as ResNet or CLIP. However the results also indicate that properties hold only in small regions. This is akin to local robustness proofs. While violations may indicate real problems in the model and/or in the specifications, often, we found that they are due to the presence of noise in the input scope $B$. To address the issue, 
we experimented with different definitions of  $B$ and also introduced the {\em slack} variables to measure the {\em degree} of violations.

While in this paper we mainly focused on checking specifications as formal explanations for DNN decisions, in the future we plan to explore other uses of $\conspec$ specifications, e.g., as run-time checks to detect mis-classifications or adversarial attacks.
An important open challenge is formally defining $B$ to only include in-distribution inputs and avoid noise. One idea is to check the specifications at run-time; $\conspec$ may be particularly useful for run-time monitoring of DNNs in safety-critical settings. 
Another direction we intend to explore is to use metrics other than cosine similarity for comparing embeddings.
We also plan to experiment with more multimodal models and to assess the effectiveness of our techniques in safety-critical applications, which have clear definition of concepts.
\section*{Acknowledgements}

This work was supported in part by the
United States Air Force and DARPA under 
Contract No.FA8750-23-C-0519, %
and the U.S. Army Research Laboratory Cooperative Research
Agreement W911NF-17-2-0196. %
Any opinions, findings and conclusions or recommendations expressed in this
material are those of the authors and do not necessarily reflect the Department of Defense
or the United States Government.

\bibliographystyle{splncs04}
\bibliography{references}

\newpage
\appendix

\section{Proofs}
\label{sec:app_proof}
Lem.~\ref{lem:eqv_verif_vals_1}, ~\ref{lem:eqv_verif_vals_2}, and Thm.~\ref{thm:eqv_verif} are all proven for a vision model $f:X\fun Y$ that can be decomposed into $\encode{f}:X\fun Z$ and $\head{f}:Z\fun Y$ where $Z:=\mathbb{R}^{d'}$, a $\conspec$ specification $e$, a linear concept representation map $\crep$ (as defined in Defn.~\ref{def:con_eqv}), and an input scope $B \subseteq X$. They use the notation $\widehat{\crep}$ for the function obtained from $\crep$ by replacing $\encode{f}$ with the identity function, $\encode{f}(B) := \{\encode{f}(x) \mid x \in B\}$, $\encode{f}^{-1}(v') := \{v \mid \encode{f}(v) = v'\}$ for the preimage under  $\encode{f}$ of a point $v'$, and IH for induction hypothesis.

A general observation that holds given the definitions of $\crep$ (as per Defn.~\ref{def:con_eqv}) and $\widehat{\crep}$ (all occurrences of $\encode{f}$ in $\crep$ replaced by the identity function) is that,
\begin{equation}
\label{eqn:lem}
\forall con \in \Concepts. \forall v \in X.~\crep(con)(v) = \widehat{\crep}(con)(\encode{f}(v)).
\end{equation}

\begin{lemma}
\label{lem:eqv_verif_vals_1}
$$
\forall v \in B.~\sem{e}(f,\crep,v) =  \sem{e}(\head{f},\widehat{\crep},\encode{f}(v))
$$
\end{lemma}
\begin{proof}

Our proof is by induction on the syntactic structure of specification $e$. 

\paragraph{\underline{Base case for $\strength(x,con_1,con_2)$:}}

\noindent For any $v \in B$, we have,

$\sem{\strength(x,con_1,con_2)}(f,\crep,v)$

$ =  \crep(con_1)(v) > \crep(con_2)(v)$

$= \widehat{\crep}(con_1)(\encode{f}(v)) > \widehat{\crep}(con_2)(\encode{f}(v))$ \hfill(Eqn.~\ref{eqn:lem})

$= \sem{\strength(x,con_1,con_2)}(\head{f},\widehat{\crep},\encode{f}(v)) $

\paragraph{\underline{Base case for $\outl(x,\labell)$:}}
\noindent For any $v \in B$, we have,

$\sem{\outl(x,\labell)}(f,\crep,v)$

$= (\argmax(f(v))=\{\labell\})$

$= (\argmax(\head{f}(\encode{f}(v))=\{\labell\})$

$= \sem{\outl(x,\labell)}(\head{f},\widehat{\crep},\encode{f}(v))$

\paragraph{\underline{Inductive case for $\neg e$:}}
For any $v \in B$, we have,

$\sem{\neg e}(f,\crep,v)$

$= \neg\sem{e}(f,\crep,v)$

$= \neg\sem{e}(\head{f},\widehat{\crep},\encode{f}(v))$\hfill (IH for $e$)

$= \sem{\neg e}(\head{f},\widehat{\crep}, \encode{f}(v))$

\paragraph{\underline{Inductive case for $e_1 \wedge e_2$:}}
For any $v \in B$, we have,

$\sem{e_1 \wedge e_2}(f,\crep,v)$

$= \sem{e_1}(f,\crep,v) \wedge \sem{e_2}(f,\crep,v)$

$= \sem{e_1}(\head{f},\widehat{\crep},\encode{f}(v)) \wedge \sem{e_2}(\head{f},\widehat{\crep},\encode{f}(v))$ \hfill (IH for $e_1,e_2$)

$= \sem{e_1 \wedge e_2}(\head{f},\widehat{\crep}, \encode{f}(v))$

\paragraph{\underline{Inductive case for $e_1 \vee e_2$:}}
For any $v \in B$, we have,

$\sem{e_1 \vee e_2}(f,\crep,v)$

$= \sem{e_1}(f,\crep,v) \vee \sem{e_2}(f,\crep,v)$

$= \sem{e_1}(\head{f},\widehat{\crep},\encode{f}(v)) \vee \sem{e_2}(\head{f},\widehat{\crep},\encode{f}(v))$ \hfill (IH for $e_1,e_2$)

$= \sem{e_1 \vee e_2}(\head{f},\widehat{\crep}, \encode{f}(v))$

$\qed$

\end{proof}

\begin{lemma}
\label{lem:eqv_verif_vals_2}
$$
\forall v' \in \encode{f}(B).\forall v \in \encode{f}^{-1}(v').~\sem{e}(\head{f},\widehat{\crep},v') = \sem{e}((f,\crep,v)
$$
\end{lemma}
\begin{proof}
Our proof is by induction on the syntactic structure of specification $e$. 

\paragraph{\underline{Base case for $\strength(x,con_1,con_2)$:}}
\noindent For any $v' \in \encode{f}(B)$, we have,

$\sem{\strength(x,con_1,con_2)}(\head{f},\widehat{\crep},v')$

$ =  \widehat{\crep}(con_1)(v') > \widehat{\crep}(con_2)(v')$

$= \widehat{\crep}(con_1)(\encode{f}(v)) > \widehat{\crep}(con_1)(\encode{f}(v)),~\forall v \in \encode{f}^{-1}(v')$ \hfill(Defn. of $\encode{f}^{-1}(v')$)

$= \crep(con_1)(v) > \crep(con_2)(v),~\forall v \in \encode{f}^{-1}(v')$ \hfill(Eqn.~\ref{eqn:lem})

$= \sem{\strength(x,con_1,con_2)}(f,\crep,v),~\forall v \in \encode{f}^{-1}(v')$

\paragraph{\underline{Base case for $\outl(x,\labell)$:}}
\noindent For any $v' \in \encode{f}(B)$, we have,

$\sem{\outl(x,\labell)}(\head{f},\widehat{\crep},v')$

$= (\argmax(\head{f}(v'))=\{\labell\})$

$= (\argmax(\head{f}(\encode{f}(v))=\{\labell\}),~\forall v \in \encode{f}^{-1}(v')$ \hfill(Defn. of $\encode{f}^{-1}(v')$)

$= (\argmax(f(v))=\{\labell\}),~\forall v \in \encode{f}^{-1}(v')$ \hfill($f = \head{f}\circ\encode{f}$)

$= \sem{\outl(x,\labell)}(f,\crep,v), ~\forall v \in \encode{f}^{-1}(v')$

\paragraph{\underline{Inductive case for $\neg e$:}}
For any $v' \in \encode{f}(B)$, we have,

$\sem{\neg e}(\head{f},\widehat{\crep},v')$

$= \neg\sem{e}(\head{f},\widehat{\crep},v')$

$= \neg\sem{e}(f,\crep,v),~\forall v \in \encode{f}^{-1}(v')$\hfill (IH for $e$)

$= \sem{\neg e}(f,\crep,v),~\forall v \in \encode{f}^{-1}(v')$

\paragraph{\underline{Inductive case for $e_1 \wedge e_2$:}}
For any $v' \in \encode{f}(B)$, we have,

$\sem{e_1 \wedge e_2}(\head{f},\widehat{\crep},v')$

$= \sem{e_1}(\head{f},\widehat{\crep},v') \wedge \sem{e_2}(\head{f},\widehat{\crep},v')$

$= \sem{e_1}(f,\crep,v) \wedge \sem{e_2}(f,\crep,v),~\forall v \in \encode{f}^{-1}(v')$ \hfill (IH for $e_1,e_2$)

$= \sem{e_1 \wedge e_2}(f,\crep,v),~\forall v \in \encode{f}^{-1}(v')$

\paragraph{\underline{Inductive case for $e_1 \vee e_2$:}}
For any $v' \in \encode{f}(B)$, we have,

$\sem{e_1 \vee e_2}(\head{f},\widehat{\crep},v')$

$= \sem{e_1}(\head{f},\widehat{\crep},v') \vee \sem{e_2}(\head{f},\widehat{\crep},v')$

$= \sem{e_1}(f,\crep,v) \vee \sem{e_2}(f,\crep,v),~\forall v \in \encode{f}^{-1}(v')$ \hfill (IH for $e_1,e_2$)

$= \sem{e_1 \vee e_2}(f,\crep,v),~\forall v \in \encode{f}^{-1}(v')$

$\qed$
\end{proof}

\noindent\textbf{Theorem~\ref{thm:eqv_verif}.}
$$
(f,\crep,B) \models e \Leftrightarrow (\head{f},\widehat{\crep},\encode{f}(B)) \models e
$$
\begin{proof}
\noindent We use the simple observation that for functions $h_1:X_1\fun X_2$, $h_2: X_2 \fun \{\truel,\falsel\}$ and $B \subseteq X_1$,
\begin{equation}
\label{eqn:th2:1}
(\forall v \in B.~h_2(h_1(v))=\truel) \Rightarrow (\forall v' \in h_1(B).~h_2(v')=\truel)
\end{equation}

\noindent We first prove the direction $(f,\crep,B) \models e \Rightarrow (\head{f},\widehat{\crep},\encode{f}(B)) \models e$.  We have that,

$(f,\crep,B) \models e$

$\forall v \in B.~\sem{e}(f,\crep,v)=\truel$

$\forall v \in B.~\sem{e}(\head{f},\widehat{\crep},\encode{f}(v))=\truel$ \hfill (Lem.~\ref{lem:eqv_verif_vals_1})

$\forall v' \in \encode{f}(B).~\sem{e}(\head{f},\widehat{\crep},v')=\truel$ \hfill (Eqn.~\ref{eqn:th2:1})

$(\head{f},\widehat{\crep},\encode{f}(B)) \models e$\\

\noindent We next prove the other direction, $(\head{f},\widehat{\crep},\encode{f}(B)) \models e \Rightarrow  (f,\crep,B) \models e$.  We have that,

$(\head{f},\widehat{\crep},\encode{f}(B)) \models e$

$\forall v' \in \encode{f}(B).~\sem{e}(\head{f},\widehat{\crep},v') = \truel$

$\forall v' \in \encode{f}(B).\forall v \in \encode{f}^{-1}(v').~\sem{e}((f,\crep,v) = \truel$ \hfill (Lem.~\ref{lem:eqv_verif_vals_2})

$\forall v \in B. ~\sem{e}((f,\crep,v) = \truel$ \hfill(since $\{v \mid v' \in \encode{f}(B) \wedge v \in \encode{f}^{-1}(v')\} = B$)

$(f,\crep,B) \models e$ 

$\qed$
\end{proof}

We state a theorem similar to Thm.~\ref{thm:eqv_verif} for a VLM model $g$.
\begin{theorem}
\label{thm:eqv_verif_VLM}
Given a VLM $g$ with encoders $\encode{g}^{img}:X\fun Z$ and $\encode{g}^{txt}:T\fun Z$ and head $\head{g}:Z\fun Y$ where $Z:=\mathbb{R}^{d'}$, a $\conspec$ specification $e$, a linear concept representation map $\crep$ (as defined in Defn.~\ref{def:con_rep_vlm}), and an input scope $B \subseteq X$,
$$
((\head{g}\circ\encode{g}^{img}),\crep,B) \models e \Leftrightarrow (\head{f},\widehat{\crep},\encode{g}(B)) \models e
$$
where $\widehat{\crep}$ is obtained from $\crep$ by replacing $\encode{g}^{img}$ with the identity function
\end{theorem}
\begin{proof}
Similar to proof for Thm.~\ref{thm:eqv_verif}.$\qed$
\end{proof}

\section{Captions Used for Computing CLIP Text Embeddings}
\label{sec:app_captions}
We use the following set of caption templates to generate captions that refer to concepts or classes. The resulting captions are then passed through CLIP's text encoder to generate text embedding. The actual captions are generated by replacing $\{\}$ with a class or concept name.

\begin{verbatim}
caption_templates = [
    'a bad photo of a {}.',
    'a photo of many {}.',
    'a photo of the hard to see {}.',
    'a low resolution photo of the {}.',
    'a rendering of a {}.',
    'a bad photo of the {}.',
    'a cropped photo of the {}.',
    'a photo of a hard to see {}.',
    'a bright photo of a {}.',
    'a photo of a clean {}.',
    'a photo of a dirty {}.',
    'a dark photo of the {}.',
    'a drawing of a {}.',
    'a photo of my {}.',
    'a photo of the cool {}.',
    'a close-up photo of a {}.',
    'a black and white photo of the {}.',
    'a painting of the {}.',
    'a painting of a {}.',
    'a pixelated photo of the {}.',
    'a bright photo of the {}.',
    'a cropped photo of a {}.',
    'a photo of the dirty {}.',
    'a jpeg corrupted photo of a {}.',
    'a blurry photo of the {}.',
    'a photo of the {}.',
    'a good photo of the {}.',
    'a rendering of the {}.',
    'a {} in an image.',
    'a photo of one {}.',
    'a doodle of a {}.',
    'a close-up photo of the {}.',
    'a photo of a {}.',
    'the {} in an image.',
    'a sketch of a {}.',
    'a doodle of the {}.',
    'a low resolution photo of a {}.',
    'a photo of the clean {}.','a photo of a large {}.',
    'a photo of a nice {}.',
    'a photo of a weird {}.',
    'a blurry photo of a {}.',
    'a cartoon {}.',
    'art of a {}.',
    'a sketch of the {}.',
    'a pixelated photo of a {}.',
    'a jpeg corrupted photo of the {}.',
    'a good photo of a {}.',
    'a photo of the nice {}.',
    'a photo of the small {}.',
    'a photo of the weird {}.',
    'the cartoon {}.',
    'art of the {}.',
    'a drawing of the {}.',
    'a photo of the large {}.',
    'a black and white photo of a {}.',
    'a dark photo of a {}.',
    'a photo of a cool {}.',
    'a photo of a small {}.',
    'a photo containing a {}.',
    'a photo containing the {}.',
    'a photo with a {}.',
    'a photo with the {}.',
    'a photo containing a {} object.',
    'a photo containing the {} object.',
    'a photo with a {} object.',
    'a photo with the {} object.',
    'a photo of a {} object.',
    'a photo of the {} object.',
    ]

\end{verbatim}

\section{Verification of CLIP}
\label{sec:app_clip}

For this experiment, we use the CLIP model $g$ itself as a classifier by adding a zero-shot classification head $\head{g}$ in the manner described in Section~\ref{sec:prelim}. For each class $\labell$, we define a region (or scope) $\overline{B}$ in the embedding space of images. Our goal is to check if the specifications formulated using the statistically significant strength predicates for a class $\labell$ of the form %
$\outl(\labell) \implies con_1\strength con_2$ hold in the region $\overline{B}$.

\paragraph{\textbf{Focus regions.}} As a proof of concept, we experimented with defining for each class $\labell$, a region $\overline{B}$ so as to approximate the set of in-distribution inputs corresponding to that class. For each image $x$ of class $\labell$ in the RIVAL10 train set ($\Dtrain$), we compute its embedding $z$, $z=\encode{g}^{img}(x)$. Then we compute the mean, $\mu_i$, and the standard deviation, $\sigma_i$, for each feature (or dimension) $i$ of the embedding space. The region $\overline{B}$ is then defined as $\overline{B} := [[l_1,u_1], \ldots, [l_p, u_p]]$, where $l_i =\mu_i  - \gamma \sigma_i$, $u_i =\mu_i  + \gamma \sigma_i$ and $\gamma$ is a parameter. Such a region is intended to contain a large number of embeddings that correspond to in-distribution images of class $\labell$. %

To validate this conjecture, we ran CLIP-guided diffusion process that generates images from CLIP embeddings.
In particular, we used 
the $\mu = [\mu_1,\ldots, \mu_p]$ as embedding for each class $\labell$ and employed the CLIP-guided diffusion toolbox~\cite{unpublished2021clip} to generate the corresponding images.
As the process is non-deterministic, we get a different sample each time we run the diffusion process. Figure~\ref{fig:clip:cars} shows a few examples for the class \textit{car}. While these images are not perfect, they have recognizable car attributes, suggesting that our regions $\overline{B}$ indeed manage to embeddings of in-distribution inputs.
\begin{figure}
    \centering
    \begin{minipage}{0.24\textwidth}
        \centering
        \includegraphics[width=0.99\textwidth]{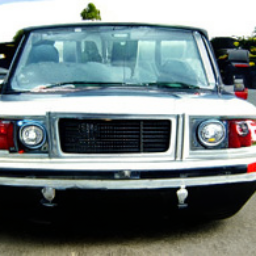} %
    \end{minipage}\hfill
    \begin{minipage}{0.24\textwidth}
        \centering
        \includegraphics[width=0.99\textwidth]{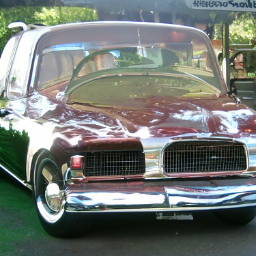} %
    \end{minipage}
    \begin{minipage}{0.24\textwidth}
        \centering
        \includegraphics[width=0.99\textwidth]{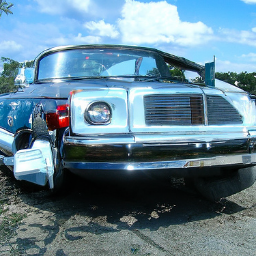} %
    \end{minipage}\hfill
    \begin{minipage}{0.24\textwidth}
        \centering
        \includegraphics[width=0.99\textwidth]{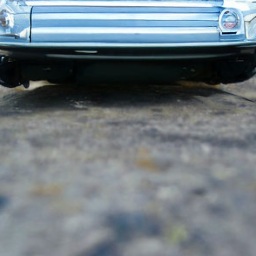} %
    \end{minipage}
    \caption{Examples of images generated using CLIP-guided stable diffusion from $\mu$ for the class \textit{car}.}\label{fig:clip:cars}
    \vspace{-0.8cm}
\end{figure}

\paragraph{\textbf{Encoding the verification problem.}}
We want to formally check if the CLIP model $(\head{g}\circ\encode{g}^{img})$ satisfies a specification $e$ with respect to a concept representation map $\crep$ and an input scope $B$ (i.e., $(\head{g}\circ\encode{g}^{img},\crep,B)\models e$) which, using Thm.~\ref{thm:eqv_verif_VLM}, can be rephrased as 
$(\head{g},\widehat{\crep},\overline{B})\models e$).

We use the zero-shot classification method as described in Section~\ref{sec:prelim} to define $\head{g}$ for determining the most probable class. We compute the mean text embedding $q^{\labell_j}$ for each class $\labell_j$ (using the CLIP text encoder $\encode{g}^{txt}$ and the set of captions listed in Appendix~\ref{sec:app_captions}) and calculate $\cosine$ similarity between an input image embedding and these classes' text embeddings. The closest class embedding is the predicted class. 
We similarly compute the embedding corresponding to each concept ${con_i}$, denoted $q^{con_i}$ here and as $\condir{con_i}$ in Defn.~\ref{def:con_rep_vlm}, in CLIP's embedding space. These embeddings represent the direction corresponding to the concepts and are used to define $\widehat{\crep}$. 

Given a specification $\outl(\labell) \implies con_1 \strength con_2$, we encode the verification problem with the negated specification, i.e., $\outl(\labell) \wedge \neg(con_1 \strength con_2)$ as follows:

\begin{align}
l_i\leq z_i \leq u_i,  l_i = \mu_i  - \gamma \sigma_i,  u_i=\mu_i  + \gamma \sigma_i, &\qquad  \forall z_i, i \in \{1,\ldots, p\},\label{eq:clip:one:box} \\ 
\sum_i \frac{z_i}{\lVert z 	\rVert}
\frac{q^{\labell}_i}{\lVert q^{\labell} 	\rVert} >  \sum_i \frac{z_i}{\lVert z 	\rVert}
\frac{q^{\labell_j}_i}{\lVert q^{\labell_j} 	\rVert}, &\qquad \forall \labell_j \neq \labell \label{eq:clip:one:class}\\
\sum_i \frac{z_i}{\lVert z \rVert} \frac{q_i^{con_2}}{\lVert q^{con_2} 	\rVert} > \sum_i \frac{z_i}{\lVert z \rVert}
\frac{q^{con_1}_i}{\lVert q^{con_1} 	\rVert}.  & \label{eq:clip:one:rel}
\end{align}

In the encoding, we use $z_i$ for the variable encoding the $i$th feature (or dimension) of embedding vector $z$.
Equation~\eqref{eq:clip:one:box} encodes the region $\overline{B}$ that we focus on and  $\gamma \in \{0.25,1,2\}$ in our experiments.
Equation~\eqref{eq:clip:one:class} encodes the zero-shot classifier ($\head{g}$) to ensure that the predicted output is class $\labell$ and Equation~\eqref{eq:clip:one:rel} encodes $\neg(con_1 \strength con_2)$. 

Note that Equations~\eqref{eq:clip:one:box}--\eqref{eq:clip:one:rel} form a non-linear model as we have $\lVert z 	\rVert$ in two sets of constraints. However, we can cancel them out and simplify them to linear constraints:
\begin{align}
\sum_i {z_i}
\frac{q^{\labell_i}}{\lVert q^{\labell} 	\rVert} >  \sum_i {z_i}
\frac{q^{\labell_j}_i}{\lVert q^{\labell_j} 	\rVert}, &\qquad \forall \labell_j \neq \labell \label{eq:clip:one:simpl:class}\\
\sum_i {z_i} \frac{q_i^{con_2}}{\lVert q^{con_2} 	\rVert} > \sum_i {z_i}
\frac{q^{con_1}_i}{\lVert q^{con_1} 	\rVert}.  & \label{eq:clip:one:simpl:rel}
\end{align}

Our results show that for considered regions $\overline{B}$, the Equations~\eqref{eq:clip:one:box},\eqref{eq:clip:one:simpl:class},\eqref{eq:clip:one:simpl:rel} are satisfiable signaling that we can find a counterexample in all cases. This potentially indicates either a problem in the CLIP model or an overly permissive definition of the input scope $B$.
To understand these results better, we reformulate the problem to an optimization problem to investigate how `strongly' we can violate the predicate $con_1 \strength con_2$. We add a slack variable in Equation~\eqref{eq:clip:one:simpl:rel} that we maximize:
\begin{align}
\sum_i {z_i} \frac{q_i^{con_2}}{\lVert q^{con_2} 	\rVert} > \varepsilon + \sum_i {z_i}
\frac{q^{con_1}_i}{\lVert q^{con_1} 	\rVert} \label{eq:clip:one:simpl:opt:rel}
\end{align}

\paragraph{\textbf{Results.}}
Figures~\ref{fig:clip:truck}--\ref{fig:clip:car} show results for our experiments where the predicted classes are \textit{truck} and \textit{car}, respectively. For each specification, we define three variants of $B$ with $\gamma \in \{0.25,1,2\}$,
and solve the optimization problem. As the size of the region increases, the strength of the violation grows significantly. For a small region, where $\gamma = 0.25\sigma$, the amount of violation is around 0.3
and it grows to about 6 when $\gamma = 2\sigma$.
We conclude that while we cannot prove that the properties formally hold, we observe the for smaller regions in the embedding space the violation is relatively low and suggests that we have to consider checking a soft version of the $con_1 \strength con_2$ predicate.

\begin{figure}[t]
    \centering
    \begin{minipage}{0.49\textwidth}
        \centering
        \includegraphics[width=0.99\textwidth]{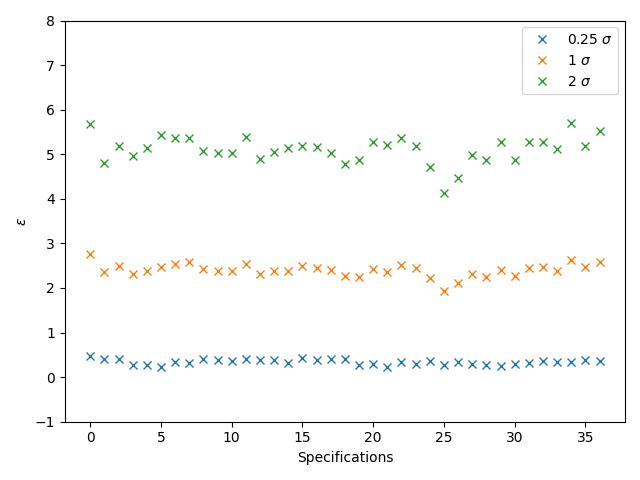} %
        \caption{The value of specification's violation measure $\varepsilon$ for  \textit{truck} (37 specs)\label{fig:clip:truck}}
    \end{minipage}\hfill
    \begin{minipage}{0.49\textwidth}
        \centering
        \includegraphics[width=0.99\textwidth]{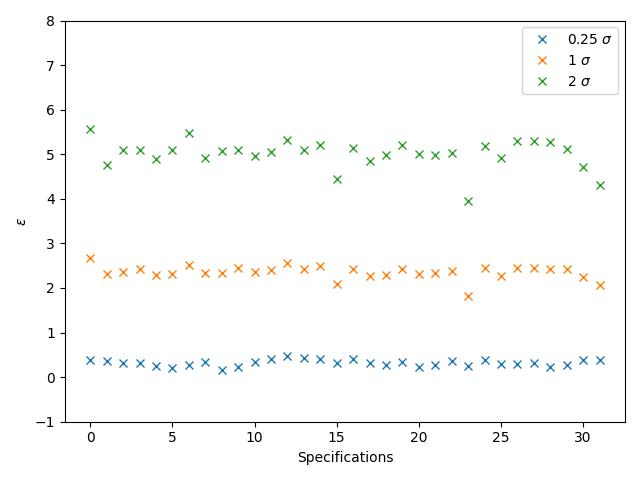} %
        \caption{The value of specification's violation measure $\varepsilon$ for \textit{car} (32 specs)\label{fig:clip:car}}
    \end{minipage}
\end{figure}

\end{document}